\documentclass{article}
\usepackage[accepted]{aistats2023}

\usepackage[round]{natbib}

\usepackage{algorithm}
\usepackage{algorithmicx}
\usepackage[noend]{algpseudocode}
\usepackage{amsfonts}
\usepackage{amsmath}
\usepackage{amssymb}
\usepackage{amsthm}
\usepackage{bbm}
\usepackage{bm}
\usepackage{color}
\usepackage{dirtytalk}
\usepackage{dsfont}
\usepackage{enumerate}
\usepackage{graphicx}
\usepackage{listings}
\usepackage{mathtools}
\usepackage{subfigure}
\usepackage{times}
\usepackage[normalem]{ulem}
\usepackage{url}
\usepackage{xspace}

\usepackage[usenames,dvipsnames]{xcolor}
\usepackage[bookmarks=false]{hyperref}
\hypersetup{
  pdffitwindow=true,
  pdfstartview={FitH},
  pdfnewwindow=true,
  colorlinks,
  linktocpage=true,
  linkcolor=Green,
  urlcolor=Green,
  citecolor=Green
}
\usepackage[capitalize,noabbrev]{cleveref}

\usepackage{tikz}
\usetikzlibrary{bayesnet}

\usepackage[textsize=tiny]{todonotes}

\usepackage{thmtools}
\declaretheorem[name=Theorem,refname={Theorem,Theorems},Refname={Theorem,Theorems}]{theorem}
\declaretheorem[name=Lemma,refname={Lemma,Lemmas},Refname={Lemma,Lemmas},sibling=theorem]{lemma}

\declaretheorem[name=Assumption,refname={Assumption,Assumptions},Refname={Assumption,Assumptions}]{assumption}

\newcommand{\cA}{\mathcal{A}}

\newcommand{\cD}{\mathcal{D}}

\newcommand{\cN}{\mathcal{N}}

\newcommand{\cS}{\mathcal{S}}

\newcommand{\cX}{\mathcal{X}}

\newcommand{\realset}{\mathbb{R}}

\newcommand{\E}[1]{\mathbb{E} \left[#1\right]}
\newcommand{\condE}[2]{\mathbb{E} \left[#1 \,\middle|\, #2\right]}
\newcommand{\Erv}[2]{\mathbb{E}_{#1} \left[#2\right]}

\newcommand{\condprob}[2]{\mathbb{P} \left(#1 \,\middle|\, #2\right)}

\newcommand{\condvar}[2]{\mathrm{var} \left[#1 \,\middle|\, #2\right]}

\newcommand{\condcov}[2]{\mathrm{cov} \left[#1 \,\middle|\, #2\right]}

\newcommand{\abs}[1]{\left|#1\right|}

\newcommand{\I}[1]{\mathds{1} \! \left\{#1\right\}}

\newcommand{\normw}[2]{\|#1\|_{#2}}

\newcommand{\set}[1]{\left\{#1\right\}}

\newcommand{\T}{^\top}

\DeclareMathOperator*{\argmax}{arg\,max\,}

\let\trace\relax
\DeclareMathOperator{\trace}{tr}
\mathchardef\mhyphen="2D

\newcommand{\flatopo}{\ensuremath{\tt FlatOPO}\xspace}
\newcommand{\hieropo}{\ensuremath{\tt HierOPO}\xspace}
\newcommand{\oracleopo}{\ensuremath{\tt OracleOPO}\xspace}

\begin{document}

\twocolumn[

\aistatstitle{Multi-Task Off-Policy Learning from Bandit Feedback}

\aistatsauthor{Joey Hong \And Branislav Kveton \And Sumeet Katariya \And Manzil Zaheer \And Mohammad Ghavamzadeh}

\aistatsaddress{UC Berkeley \And Amazon \And Amazon \And Deepmind \And Google}]

\begin{abstract}
Many practical applications, such as recommender systems and learning to rank, involve solving multiple similar tasks. One example is learning of recommendation policies for users with similar movie preferences, where the users may still rank the individual movies slightly differently. Such tasks can be organized in a hierarchy, where similar tasks are related through a shared structure. In this work, we formulate this problem as a contextual off-policy optimization in a hierarchical graphical model from logged bandit feedback. To solve the problem, we propose a hierarchical off-policy optimization algorithm (\hieropo), which estimates the parameters of the hierarchical model and then acts pessimistically with respect to them. We instantiate \hieropo in linear Gaussian models, for which we also provide an efficient implementation and analysis. We prove per-task bounds on the suboptimality of the learned policies, which show a clear improvement over not using the hierarchical model. We also evaluate the policies empirically. Our theoretical and empirical results show a clear advantage of using the hierarchy over solving each task independently.
\end{abstract}

\section{Introduction}
\label{sec:introduction}

Many interactive systems (search, online advertising, and recommender systems) can be modeled as a \emph{contextual bandit} \citep{li10contextual,chu11contextual}, where an agent, or \emph{policy}, observes a \emph{context}, takes one of $K$ possible \emph{actions}, and receives a \emph{stochastic reward} for the action. In many applications, it is prohibitively expensively to learn policies online by contextual bandit algorithms, because exploration has a major impact on user experience. However, offline data collected by a previously deployed policy are often available. Offline, or \emph{off-policy}, optimization using such logged data is a practical way of learning policies without costly online interactions \citep{dudik14doubly,swaminathan2015counterfactual}.

Because we cannot explore beyond the logged dataset, it is critical to design learning algorithms that use the data in the most efficient way. One way of achieving this is by leveraging the structure of the problem. As an example, in bandit algorithms, we could achieve higher statistical efficiency by using the form of the reward distribution \citep{garivier11klucb}, prior distribution over model parameters \citep{thompson33likelihood,agrawal12analysis,chapelle11empirical,russo18tutorial}, or by conditioning on feature vectors \citep{dani08stochastic,abbasi-yadkori11improved,agrawal13thompson}. In this work, we consider a natural structure where we design policies for multiple similar tasks, where the tasks are related through a \emph{hierarchical Bayesian model} \citep{gelman13bayesian,kveton21metathompson,hong2022hierarchical}. Each task is parameterized by a \emph{task parameter} sampled i.i.d.\ from a distribution parameterized by a \emph{hyper-parameter}. These parameters are unknown and relate the tasks, in the sense that data from one task can help with learning a policy for another task.

Although the tasks are similar, they are sufficiently different to require different polices, and we address this multi-task off-policy learning problem in this work. To solve the problem, we propose an algorithm called hierarchical off-policy optimization (\hieropo). Because off-policy algorithms must reason about counterfactual rewards of actions that do not appear in the logged dataset, a common approach is to learn pessimistic, or \emph{lower confidence bound (LCB)}, estimates of the mean rewards and act according to them \citep{buckman2020importance,jin2021pessimism}. \hieropo is an instance of this approach where high-probability LCBs are estimated using a hierarchical model.

Our paper makes the following contributions. First, we discuss how hierarchy can improve statistical efficiency, which motivates our algorithm \hieropo. The key idea in \hieropo is to factorize the computation of LCBs by separately considering the uncertainty of the hyper-parameter and the conditional uncertainty of task parameters. Second, we consider a specific hierarchical model, a linear Gaussian model, where we obtain closed forms for the LCBs that can be computed efficiently. Third, we derive Bayesian suboptimality bounds for the policies learned by \hieropo and show that they improve upon off-policy approaches that do not use the hierarchy. To the best of our knowledge, we are the first to consider Bayesian bounds in the off-policy setting. Finally, we evaluate \hieropo on synthetic problems and an application to a multi-user recommendation system.

\section{Setting}
\label{sec:setting}

\textbf{Notation.} Random variables are capitalized, except for Greek letters like $\theta$. For any positive integer $n$, we define $[n] = \set{1, \dots, n}$. The indicator function is denoted by $\I{\cdot}$. The $i$-th entry of vector $v$ is $v_i$. If the vector is already indexed, such as $v_j$, we write $v_{j, i}$. For any matrix $M \in \realset^{d \times d}$, the maximum and minimum eigenvalues are $\lambda_1(M)$ and $\lambda_d(M)$, respectively.

We consider a learning agent that interacts with a set of contextual bandit instances. In each interaction, the agent observes a \emph{context} $x \in \cX$, takes an \emph{action} $a$ from an \emph{action set} $\cA$ of size $K$, and then observes a \emph{stochastic reward} $Y \in \realset$. The contexts are sampled from the \emph{context distribution} $P_\mathsf{x}$. Conditioned on context and action, the reward is sampled from the \emph{reward distribution} $P(\cdot \mid x, a; \theta)$, where $\theta \in \Theta$ is a parameter of the bandit instance, which is \emph{shared by all contexts and actions}. We assume that the rewards are $\sigma^2$-sub-Gaussian and denote by $r(x, a; \theta) = \Erv{Y \sim P(\cdot \mid x, a; \theta)}{Y}$ the mean reward of action $a$ in context $x$ under parameter $\theta$.

In this work, the learning agent simultaneously solves $m$ contextual bandit instances, which we denote by $\cS = [m]$ and refer to as \emph{tasks}. Therefore, we call our problem a \emph{multi-task contextual bandit} \citep{azar13sequential,deshmukh17multitask,cella20metalearning,kveton21metathompson,moradipari21parameter}. Each task $s \in \cS$ is parameterized by a \emph{task parameter} $\theta_{s, *} \in \Theta$, which is sampled i.i.d.\ from a \emph{task prior distribution} $\theta_{s, *} \sim P(\cdot \mid \mu_*)$. The task prior is parameterized by an unknown \emph{hyper-parameter} $\mu_*$, which is sampled from a \emph{hyper-prior} $Q$. That one is known to the agent and represents its prior knowledge about $\mu_*$. In a recommender system, each task could be an individual user, the task parameter could encode user's preferences, and the hyper-parameter could encode the average preferences of a cluster of similar users. We use this setup in our experiments in \cref{sec:experiments}. A similar setup was studied previously in the online setting by \citet{hong22hierarchical}.

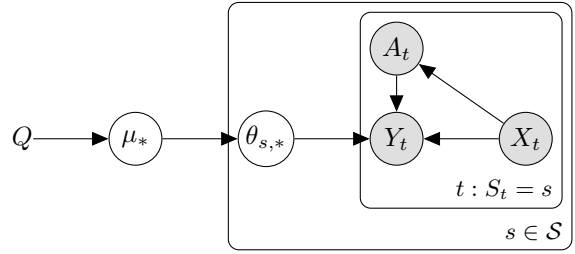
\begin{figure}
  \begin{tikzpicture}

  \node[obs] (x) {$X_{t}$};
  \node[obs, left=of x] (y) {$Y_{t}$};
  \node[obs, left=of x, yshift=12mm] (a) {$A_{t}$};
  \node[latent, left=of y] (theta) {$\theta_{s, *}$};
  \node[latent, left=of theta]  (mu) {$\mu_*$};
  \node[const, left=of mu] (q) {$Q$};
  \edge {q} {mu} ;
  \edge {mu} {theta} ; %
  \edge {theta} {y} ; %
  \edge {x} {y} ;
  \edge {a} {y} ;
  \edge {x} {a} ;

  \plate {y1} { (a) (x) (y)} {$t: S_t = s $} ;
  \plate {theta} {(theta) (y1)} {$s\in \cS $} ;

\end{tikzpicture}
  \caption{A graphical model of our multi-task contextual bandit setting.}
  \label{fig:graphical model}
  \vspace{-0.1in}
\end{figure}

Unlike prior works in multi-task bandits, we aim to solve this problem offline. Let $\Pi = \{\pi: \cX \to \cA\}$ be the set of \emph{stationary deterministic policies}. For any policy $\pi$ and context $x$, we denote by $\pi(x)$ the action suggested by $\pi$ in context $x$. In our multi-task bandit setting, each task has its own parameter, and thus we may need a different policy to solve it. Therefore, we consider the set of \emph{task-conditioned policies} $\pi \in \Pi^\cS = \{(\pi_s)_{s \in \cS}: \pi_s \in \Pi\}$, where $\pi_s$ is the policy for task $s$. Note that we consider deterministic policies solely to simplify notation, and that our results extend to stochastic policies by accounting for an additional expectation over actions.

A logged dataset of past interactions is an input to off-policy evaluation and optimization. In our setting, we have access to a dataset $\cD = \{(S_t, X_t, A_t, Y_t)\}_{t \in [n]}$ of $n$ observations, where $S_t \in \cS$ is a task, $X_t \sim P_\mathsf{x}$ is a context, $A_t = \pi_{0, S_t}(X_t)$ is an action, and $Y_t \sim P(\cdot \mid X_t, A_t; \theta_{S_t, *})$ is a reward in observation $t$. Here $\pi_0 \in \Pi^\cS$ is a \emph{logging policy}, some task-conditioned policy that is used to collect $\cD$. 
A graphical model of our setting is shown in \cref{fig:graphical model}.
Unlike many works in off-policy learning, we do not require that $\pi_0$ is known \citep{dudik14doubly,swaminathan2015counterfactual}.

The \emph{value of policy} $\pi_s \in \Pi$ in task $s \in \mathcal{S}$ with parameter $\theta_{s, *}$ is defined as
\begin{align*}
  V(\pi_s; \theta_{s, *})
  = \condE{r(X, \pi_s(X); \theta_{s, *})}{\theta_{s, *}}\,,
\end{align*}
where the randomness is only over context $X \sim P_\mathsf{x}$. The \emph{optimal policy} $\pi_{s, *}$ is defined as
\begin{align*}
  \pi_{s, *}
  = \argmax_{\pi \in \Pi} V(\pi; \theta_{s, *})
\end{align*}
and the \emph{suboptimality} of policy $\pi_s$ is
\begin{align*}
  V(\pi_{s, *}; \theta_{s, *}) - V(\pi_s; \theta_{s, *})\,.
\end{align*}
We study the Bayesian setting, where the logged dataset $\cD$ provides additional information about the parameter $\theta_{s, *}$. In particular, let $\hat{P}_s(\theta) = \condprob{\theta_{s, *} = \theta}{\cD}$ be the posterior distribution of $\theta_{s, *}$ in task $s$ given $\cD$. Then, by definition, $\theta_{s, *} \mid \cD \sim \hat{P}_s$. Our goal is to learn a policy, for any given task $s$, that is comparable to likely $\pi_{s, *} \mid \cD$. We formalize this objective using a high-probability bound. For a \emph{fixed confidence level} $\delta \in (0, 1)$, we want to learn a policy $\hat{\pi}_s \in \Pi$ that minimizes $\varepsilon$ in
\begin{align}
  \condprob{V(\pi_{s, *}; \theta_{s, *}) - V(\hat{\pi}_s; \theta_{s, *})
  \leq \varepsilon}{\cD}
  \geq 1 - \delta\,,
  \label{eq:multi-task suboptimality}
\end{align}
where $\varepsilon$ is a function of $\delta$, the environment parameters, $\cD$, and $\hat{\pi}_s$. Note that $\pi_{s, *}$ is random because it is a function of random $\theta_{s, *} \mid \cD \sim \hat{P}_s$.

The Bayesian view allows us to derive error bounds with two new properties. First, the error $\varepsilon$ decreases with a more informative prior on $\theta_{s, *}$. Second, the bounds capture the structure of our hierarchical problem and show that it helps. Although our objective and analysis style are novel, they are motivated by Bayes regret bounds in bandits \citep{russo14learning,lu19informationtheoretic,kveton21metathompson,hong22hierarchical}, which have similar properties that allow them to improve upon their frequentist counterparts \citep{abbasi-yadkori11improved,agrawal13thompson}.

\section{Algorithm}
\label{sec:algorithm}

Prior works in off-policy bandit and reinforcement learning often design pessimistic lower confidence bounds and then act on them \citep{jin2021pessimism}. We follow the same design principle. For any task $s$, context $x$, and action $a$, we want to estimate a LCB satisfying $L_s(x, a) \leq r(x, a; \theta_{s, *})$, with a high probability for $\theta_{s, *} \mid \cD$. We seek the LCBs of the form $L_s(x, a) = \hat{r}_s(x, a) - c_s(x, a)$, where
\begin{equation}
\begin{aligned}
  \hat{r}_s(x, a)
  & = \condE{r(x, a; \theta_{s, *})}{\cD}\,,
  \label{eq:lcb components} \\
  c_s(x, a)
  & = \alpha \sqrt{\condvar{r(x, a; \theta_{s, *})}{\cD}}\,,
\end{aligned}
\end{equation}
are the estimated mean reward and its confidence interval width, and $\alpha > 0$ is a tunable parameter.

An important case of contextual models are those with linear rewards \citep{abbasi-yadkori11improved,jin2021pessimism}. In our paper, we assume that $r(x, a; \theta_{s, *}) = \phi(x, a)\T \theta_{s, *}$ for each task $s$, where $ \theta_{s, *}$ is the task parameter and $\phi: \cX \times \cA \to \mathbb{R}^d$ is some \emph{feature extractor}. Under this assumption, we may write \eqref{eq:lcb components} using the posterior mean and covariance of $\theta_{s, *}$ as
\begin{equation}
\begin{aligned}
  \hat{r}_s(x, a)
  & = \phi(x, a)\T \condE{\theta_{s, *}}{\cD}\,,
  \label{eq:lcb components linear} \\
  c_s(x, a)
  & = \alpha \sqrt{\phi(x, a)\T \condcov{\theta_{s, *}}{\cD} \phi(x, a)}\,.
\end{aligned}
\end{equation}
The above is desirable because it separates the posterior of the task parameter from context.

The rest of this section is organized as follows. In \cref{sec:hierarchical pessimism}, we derive the mean reward estimate and its confidence interval width for a general two-level hierarchical model. We also propose a general hierarchical off-policy optimization (\hieropo) in this model. In \cref{sec:hierarchical gaussian pessimism}, we instantiate this model as a linear Gaussian model. We discuss alternative algorithm designs in \cref{sec:alternative designs}.

\subsection{Hierarchical Pessimism}
\label{sec:hierarchical pessimism}

For any task $s$, the mean $\condE{\theta_{s, *}}{\cD}$ in \eqref{eq:lcb components linear} can be estimated hierarchically as follows. Let $\cD_s$ be the subset of dataset $\cD$ corresponding to task $s$. By the law of total expectation,
\begin{align}
  \condE{\theta_{s, *}}{\cD}
  & = \condE{\condE{\theta_{s, *}}{\mu_*, \cD}}{\cD}
  \label{eq:posterior task mean} \\
  & = \condE{\condE{\theta_{s, *}}{\mu_*, \cD_s}}{\cD}\,.
  \nonumber
\end{align}
The second equality holds since conditioning on $\mu_*$ makes $\theta_{s, *}$ independent of $\cD \setminus \cD_s$, as can be seen in \cref{fig:graphical model}. The above decomposition is motivated by the observation that estimating each $\condE{\theta_{s, *}}{\mu_*, \cD_s}$ is an easier problem than $\condE{\theta_{s, *}}{\cD}$, since $\cD_s$ is from a single task $s$. The information sharing between the tasks is still captured by $\mu_*$, which has to be learned from the entire logged dataset $\cD$.

Similarly, the covariance $\condcov{\theta_{s, *}}{\cD}$ in \eqref{eq:lcb components linear} can be decomposed using the law of total covariance,
\begin{align}
  & \condcov{\theta_{s, *}}{\cD}
  \label{eq:posterior task covariance} \\
  & \ = \condE{\condcov{\theta_{s, *}}{\mu_*, \cD}}{\cD} +
  \condcov{\condE{\theta_{s, *}}{\mu_*, \cD}}{\cD}
  \nonumber \\
  & \ = \condE{\condcov{\theta_{s, *}}{\mu_*, \cD_s}}{\cD} +
  \condcov{\condE{\theta_{s, *}}{\mu_*, \cD_s}}{\cD}\,.
  \nonumber
\end{align}
Again, the second equality holds since conditioning on $\mu_*$ makes $\theta_{s, *}$ independent of $\cD \setminus \cD_s$. Note that \eqref{eq:posterior task covariance} comprises two interpretable terms. The first captures the uncertainty of $\theta_{s, *}$ conditioned on $\mu_*$, whereas the second captures the uncertainty in $\mu_*$. Such decompositions decouple the two sources of uncertainty in our hierarchical model, and are powerful tools for estimating uncertainty in structured models \citep{hong22deep}.

Now we plug \eqref{eq:posterior task mean} and \eqref{eq:posterior task covariance} into \eqref{eq:lcb components linear}, and get
\begin{align*}
  \hat{r}_s(x, a)
  & = \phi(x, a)\T \condE{\condE{\theta_{s, *}}{\mu_*, \cD_s}}{\cD}\,, \\
  c_s(x, a)
  & = \alpha \sqrt{\phi(x, a)\T \hat{\Sigma}_s \phi(x, a)}\,,
\end{align*}
where
\begin{align*}
  \hat{\Sigma}_s
  = \condE{\condcov{\theta_{s, *}}{\mu_*, \cD_s}}{\cD} +
  \condcov{\condE{\theta_{s, *}}{\mu_*, \cD_s}}{\cD}\,.
\end{align*}
With this in mind, we propose a general algorithm for hierarchical off-policy optimization, which we call \hieropo and report its pseudo-code in \cref{alg:hieropo}.

\begin{algorithm}[t]
  \caption{\hieropo: Hierarchical off-policy optimization.}
  \label{alg:hieropo}
  \begin{algorithmic}[1]
    \State \textbf{Input:} Dataset $\cD$
    \For{$s \in \cS, x \in \cX$}
      \For{$a \in \cA$}
        \State Compute $\hat{r}_s(x, a)$ and $c_s(x, a)$ (\cref{sec:hierarchical pessimism})
        \State $L_s(x, a) \gets \hat{r}_s(x, a) - c_s(x, a)$
        \State $\hat{\pi}_s(x) \gets \argmax_{a \in \cA} L_s(x, a)$
      \EndFor
    \EndFor
    \State \textbf{Output:} $\hat{\pi} \gets (\hat{\pi}_s)_{s \in \cS}$
  \end{algorithmic}
\end{algorithm}

\subsection{Hierarchical Gaussian Pessimism}
\label{sec:hierarchical gaussian pessimism}

The computation of \eqref{eq:posterior task mean} and \eqref{eq:posterior task covariance} requires integrating out the hyper-parameter $\mu_*$ and task parameter $\theta_{s, *}$. This is generally impossible in a closed form, although many powerful approximations exist \citep{doucet01sequential}. In this section, we consider the case where the hyper-prior and task prior distributions are Gaussian. In this case, \hieropo can be implemented exactly and efficiently. The later analysis of \hieropo (\cref{sec:multi-task analysis}) is also under this assumption.

Specifically, we consider a linear Gaussian model where the known hyper-prior is $Q = \cN(\mu_q, \Sigma_q)$ for some PSD matrix $\Sigma_q$ and the task prior is $P(\cdot \mid \mu_*) = \cN(\mu_*, \Sigma_0)$ for some known PSD $\Sigma_0$. The reward distribution of action $a$ in context $x$ is $\cN(\phi(x, a)\T \theta_{s, *}, \sigma^2)$, where $\phi$ is a feature extractor and $\sigma > 0$ is a known reward noise. This implies that the mean reward is linear in features.

To derive \eqref{eq:posterior task mean} and \eqref{eq:posterior task covariance}, we start with understanding posterior distributions of $\theta_{s, *}$ and $\mu_*$. Specifically, since conditioning in Gaussian graphical models preserves Gaussianity, we have that $\theta_{s, *} \mid \mu_*, \cD_s \sim \cN(\tilde{\mu}_s, \tilde{\Sigma}_s)$ for some $\tilde{\mu}_s$ and $\tilde{\Sigma}_s$. From the structure of our model (\cref{fig:graphical model}), we further note that this is a standard posterior of a linear model with a Gaussian prior $\cN(\mu_*, \Sigma_0)$, and thus,
\begin{equation}
\begin{aligned}
  \tilde{\mu}_s
  & = \condE{\theta_{s, *}}{\mu_*, \cD_s}
  = \tilde{\Sigma}_s (\Sigma_0^{-1} \mu_* + B_s)\,,
  \label{eq:linear task posterior} \\
  \tilde{\Sigma}_s
  & = \condcov{\theta_{s, *}}{\mu_*, \cD_s}
  = (\Sigma_0^{-1} + G_s)^{-1}\,,
\end{aligned}
\end{equation}
where the statistics
\begin{align}
  B_s
  & = \sigma^{-2} \sum_{t = 1}^n \I{S_t = s} \phi(X_t, A_t) Y_t\,,
  \nonumber \\
  G_s
  & = \sigma^{-2} \sum_{t = 1}^n \I{S_t = s} \phi(X_t, A_t) \phi(X_t, A_t)\T\,,
  \nonumber
\end{align}
are computed using the subset $\cD_s$ of the logged dataset $\cD$.

The posterior of the hyper-parameter $\mu_* \mid \cD$, known as the hyper-posterior, also has a closed-form $\cN(\bar{\mu}, \bar{\Sigma})$ (Section 4.2 of \citealt{hong22hierarchical}), where
\begin{equation}
\begin{aligned}
  \bar{\mu}
  & = \condE{\mu_*}{\cD}
  \label{eq:linear hyper-posterior} \\
  & = \bar{\Sigma} \Big(\Sigma_q^{-1} \mu_q +
  \sum_{s \in \cS} (\Sigma_0 + G_s^{-1})^{-1} G_s^{-1} B_s\Big)\,, \\
  \bar{\Sigma}
  & = \condcov{\mu_*}{\cD}
  = \Big(\Sigma_q^{-1} +
  \sum_{s \in \cS} (\Sigma_0 + G_s^{-1})^{-1}\Big)^{-1}.
\end{aligned}
\end{equation}
It is helpful to view \eqref{eq:linear hyper-posterior} as a multivariate Gaussian posterior where each task is a single observation. The observation of task $s$ is the least-squares estimate of $\theta_{s, *}$ from task $s$, $G_s^{-1} B_s$, and its covariance is $\Sigma_0 + G_s^{-1}$. The tasks with many observations affect the value of $\bar{\mu}$ more, because their $G_s^{-1}$ approaches a zero matrix. In this case, $\Sigma_0 + G_s^{-1} \to \Sigma_0$. This uncertainty cannot be reduced because even $\theta_{s, *}$ is a noisy observation of $\mu_*$ with covariance $\Sigma_0$.

To complete our derivations, we only need to substitute \eqref{eq:linear task posterior} and \eqref{eq:linear hyper-posterior} into \eqref{eq:posterior task mean} and \eqref{eq:posterior task covariance}. The posterior mean of $\theta_{s, *}$ is
\begin{align*}
  \condE{\condE{\theta_{s, *}}{\mu_*, \cD_s}}{\cD}
  & = \condE{\tilde{\Sigma}_s (\Sigma_0^{-1} \mu_* + B_s)}{\cD} \\
  & = \tilde{\Sigma}_s (\Sigma_0^{-1} \condE{\mu_*}{\cD} + B_s) \\
  & = \tilde{\Sigma}_s (\Sigma_0^{-1} \bar{\mu} + B_s)\,,
\end{align*}
where we simply combine \eqref{eq:linear task posterior} and \eqref{eq:linear hyper-posterior}. Similarly, the posterior covariance of $\theta_{s, *}$ requires computing
\begin{align*}
  \condE{\condcov{\theta_{s, *}}{\mu_*, \cD_s}}{\cD}
  & = \condE{\tilde{\Sigma}_s}{\cD}
  = \tilde{\Sigma}_s\,, \\
  \condcov{\condE{\theta_{s, *}}{\mu_*, \cD_s}}{\cD}
  & = \condcov{\tilde{\Sigma}_s (\Sigma_0^{-1} \mu_* + B_s)}{\cD} \\
  & = \condcov{\tilde{\Sigma}_s \Sigma_0^{-1} \mu_*}{\cD} \\
  & = \tilde{\Sigma}_s \Sigma_0^{-1} \bar{\Sigma} \Sigma_0^{-1} \tilde{\Sigma}_s\,.
\end{align*}
Finally, the estimated mean reward and its confidence interval width are given by
\begin{equation}
\begin{aligned}
  \hat{r}_s(x, a)
  & = \phi(x, a)\T \tilde{\Sigma}_s (\Sigma_0^{-1} \bar{\mu} + B_s)\,, \\
  c_s(x, a)
  & = \alpha \sqrt{\phi(x, a)\T \hat{\Sigma}_s \phi(x, a)}\,,
  \label{eq:lcb}
\end{aligned}
\end{equation}
where $\hat{\Sigma}_s = \tilde{\Sigma}_s + \tilde{\Sigma}_s \Sigma_0^{-1} \bar{\Sigma} \Sigma_0^{-1} \tilde{\Sigma}_s$. Note that the posterior covariance $\hat{\Sigma}_s$ can be computed tractably, and exhibits the following desirable properties. First, the uncertainty over the hyper-parameter only shows up in the second term in $\bar{\Sigma}$. In addition, since $\tilde{\Sigma}_s$ appears in both terms, both terms become smaller with more observations from task $s$.

\subsection{Alternative Designs}
\label{sec:alternative designs}

A natural question to ask is what is the benefit of leveraging hierarchy in obtaining pessimistic reward estimates. To answer this question, we compare \hieropo in \cref{sec:hierarchical gaussian pessimism} to two alternative algorithms. The first one is unrealistic and assumes that $\mu_*$ is known. We call it \oracleopo. In this case, the posterior mean reward and its confidence interval width are given by
\begin{align*}
  \hat{r}_s(x, a) 
  & = \phi(x, a)\T \tilde{\Sigma}_s (\Sigma_0^{-1} \mu_* + B_s)\,, \\
  c_s(x, a)
  & = \alpha \sqrt{\phi(x, a)\T \tilde{\Sigma}_s \phi(x, a)}\,.
\end{align*}
This improves upon \eqref{eq:lcb} in two aspects. First, the estimate $\bar{\mu}$ of $\mu_*$ is replaced with the actual $\mu_*$. Second, the confidence interval width is provably smaller because
\begin{align*}
  \tilde{\Sigma}_s
  \preceq 
    \tilde{\Sigma}_s + \tilde{\Sigma}_s \Sigma_0^{-1} \bar{\Sigma} \Sigma_0^{-1} \tilde{\Sigma}_s\,.
\end{align*}
In the second algorithm, we consider what happens when we do not model the hierarchy, which we dub \flatopo. In this case, we do not attempt to model $\mu_*$ and include its uncertainty in $\theta_{s, *}$. To do so, the conditional uncertainty of $\theta_{s, *}$, represented by $\Sigma_0$, is replaced with its marginal uncertainty, represented by $\Sigma_q + \Sigma_0$. As a result, the posterior mean reward and its confidence interval width are
\begin{align*}
  \hat{r}_s(x, a) 
  & = \phi(x, a)\T \dot{\Sigma}_s ((\Sigma_q + \Sigma_0)^{-1} \mu_q + B_s)\,, \\
  c_s(x, a)
  & = \alpha \sqrt{\phi(x, a)\T \dot{\Sigma}_s \phi(x, a)}\,,
\end{align*}
where $\dot{\Sigma}_s = ((\Sigma_q + \Sigma_0)^{-1} + G_s)^{-1}$. This is worse than \eqref{eq:lcb} in two aspects. First, the prior mean $\mu_q$ of $\mu_*$ is used instead of its estimate $\bar{\mu}$. Second, as the number of tasks $m$ increases,
\begin{align*}
  \dot{\Sigma}_s
  \succeq \tilde{\Sigma}_s \Sigma_0^{-1} \bar{\Sigma} \Sigma_0^{-1} \tilde{\Sigma}_s +
  \tilde{\Sigma}_s\,,
\end{align*}
since $\bar{\Sigma}$ in \eqref{eq:linear hyper-posterior} approaches a zero matrix. Therefore, our approach should be more statistically efficient, which we prove formally in \cref{sec:multi-task analysis}.

\section{Single-Task Analysis}
\label{sec:single-task analysis}

To illustrate our error bounds, we start with a contextual bandit parameterized by $\theta_* \in \realset^d$. The mean reward of action $a \in \cA$ in context $x \in \cX$ under parameter $\theta \in \realset^d$ is $r(x, a; \theta) = \phi(x, a)\T \theta$. We assume that $\theta_* \sim \cN(\theta_0, \Sigma_0)$ and that the reward noise is $\cN(0, \sigma^2)$. Note that this is an analogous model to a single task in \cref{sec:hierarchical gaussian pessimism} where we drop indexing by $s$ to simplify notation.

The logged dataset is $\cD = \set{(X_t, A_t, Y_t)}_{t = 1}^n$, the LCB is $L(x, a) = \hat{r}(x, a) - c(x, a)$, and we output a policy $\hat{\pi} \in \Pi$ defined as $\hat{\pi}(x) = \argmax_{a \in \cA} L(x, a)$. Following the same reasoning as in the derivation of \eqref{eq:lcb}, the estimated mean reward and its confidence interval width are
\begin{align*}
  \hat{r}(x, a)
  & = \phi(x, a)\T \hat{\Sigma} (\Sigma_0^{-1} \theta_0 + B)\,, \\
  c(x, a)
  & = \alpha \sqrt{\phi(x, a)\T \hat{\Sigma} \phi(x, a)}\,,
\end{align*}
where
\begin{align*}
  \hat{\Sigma}
  & = (\Sigma_0^{-1} + G)^{-1}\,, \\
  B
  & = \sigma^{-2} \sum_{t = 1}^n \phi(X_t, A_t) Y_t\,, \\
  G
  & = \sigma^{-2} \sum_{t = 1}^n \phi(X_t, A_t) \phi(X_t, A_t)\T\,.
\end{align*}
Analogously to \cref{sec:setting}, the value of policy $\pi \in \Pi$ under parameter $\theta_*$ is $V(\pi; \theta_*) = \condE{r(X, \pi(X); \theta_*)}{\theta_*}$ and the optimal policy is $\pi_* = \argmax_{\pi \in \Pi} V(\pi; \theta_*)$. For any fixed confidence level $\delta > 0$, our goal is to learn a policy $\hat{\pi} \in \Pi$ that minimizes $\varepsilon$ in
\begin{align}
  \condprob{V(\pi_*; \theta_*) - V(\hat{\pi}; \theta_*)
  \leq \varepsilon}{\cD} \geq 1 - \delta\,.
  \label{eq:single-task suboptimality}
\end{align}
We make the following assumptions in our analysis. First, we assume that the length of feature vectors is bounded.

\begin{assumption}
\label{ass:bounded feature vectors} For any $x \in \cX$ and $a \in \cA$, the feature vector satisfies $\normw{\phi(x, a)}{2} \leq 1$.
\end{assumption}

This assumption is without loss of generality and only simplifies presentation. Second, similarly to prior works \citep{swaminathan2017off,jin2021pessimism}, we assume that the dataset $\cD$ is \say{well-explored}.

\begin{assumption}
\label{ass:precision lower bound} Let
\begin{align*}
  G_*
  = \condE{\phi(X, \pi_*(X)) \phi(X, \pi_*(X))\T}{\theta_*}\,.
\end{align*}
Then there exists $\gamma > 0$ such that $G \succeq \gamma \sigma^{-2} n G_*$ holds for any $\theta_*$.
\end{assumption}

The above assumption relates the logging policy $\pi_0$, which defines the empirical precision $G$, to the optimal policy $\pi_*$, which defines the mean precision $\sigma^{-2} n G_*$. The assumption can be loosely interpreted as follows. As $n$ increases, $G \to \sigma^{-2} n \E{\phi(X, \pi_0(X)) \phi(X, \pi_0(X))\T}$, and hence $\gamma$ can be viewed as the maximum ratio between probabilities of taking actions by $\pi_*$ and $\pi_0$ in any direction. In general, for a uniform logging policy, $\gamma = \Omega(1 / d)$ when $n$ is large. The assumption essentially allows us not to reason about the properties of $G$ when $n$ is small, which would require a concentration argument and is not essential to our result.

Note that the assumption is always satisfied by setting $\gamma = 0$. However, this setting would negate the desired scaling with sample size $n$ in our error bounds. Also note that the assumption can be weakened to be probabilistic over $\theta_*$. We do not do this to simplify the exposition.

Now we state our main claim for the single-task setting.

\begin{theorem}
\label{thm:single-task bound} Fix dataset $\cD$ and choose any $\gamma$ such that \cref{ass:precision lower bound} holds. Let $\hat{\pi}(x) = \argmax_{a \in \cA} L(x, a)$. Then for any $\delta \in (0, 1)$ and
\begin{align*}
  \alpha
  = \sqrt{5 d \log(1 / \delta)}\,,
\end{align*}
the suboptimality of $\hat{\pi} \in \Pi$ in \eqref{eq:single-task suboptimality} is bounded for
\begin{align*}
  \varepsilon
  = \alpha \sqrt{\frac{4 d}{\lambda_d(\Sigma_0^{-1}) + \gamma \sigma^{-2} n}}\,.
\end{align*}
\end{theorem}
\begin{proof}
The claim is proved in \cref{sec:single-task proof} in three steps. First, we establish that $c(x, a)$ is a high-probability confidence interval width for $\alpha = \sqrt{5 d \log(1 / \delta)}$. Second, we show that the suboptimality of policy $\hat{\pi}$ can be bounded by $\condE{c(X, \pi_*(X))}{\theta_*}$. Finally, we combine closed forms of $c(x, a)$ with \cref{ass:precision lower bound}, and relate the statistics under the logging policy $\pi_0$ that define $c(x, a)$ with the expectation under $\pi_*$.
\end{proof}

\section{Multi-Task Analysis}
\label{sec:multi-task analysis}

Now we study our multi-task setting, where the estimated mean reward and its confidence interval width are defined in \eqref{eq:lcb}. Similarly to \cref{sec:single-task analysis}, this analysis is Bayesian and we are concerned with the distribution of model parameters conditioned on $\cD$. We fix the task and derive an error bound for a single $s \in \cS$. In \cref{sec:discussion}, we discuss how to extend our bound to other performance metrics, such as the error over all tasks.

To derive the bound in \eqref{eq:multi-task suboptimality}, we make assumptions analogous to \cref{sec:single-task analysis}. First, and without loss of generality, we assume that the length of feature vectors is bounded (\cref{ass:bounded feature vectors}). Second, we assume that the dataset $\cD$ is \say{well-explored} for all tasks.

\begin{assumption}
\label{ass:multi-task precision lower bound} Let
\begin{align*}
  G_s
  = \sigma^{-2} \sum_{t = 1}^n \I{S_t = s} \phi(X_t, A_t) \phi(X_t, A_t)\T
\end{align*}
be the empirical precision associated with task $s$ and $n_s = \sum_{t = 1}^n \I{S_t = s}$ be the number of interactions with that task. Let
\begin{align*}
  G_{s, *}
  = \condE{\phi(X, \pi_{s, *}(X)) \phi(X, \pi_{s, *}(X))\T}{\theta_{s, *}}\,.
\end{align*}
Then there exists $\gamma > 0$ such that $G_s \succeq \gamma \sigma^{-2} n_s G_{s, *}$ holds for any $\theta_{s, *}$ in any task $s \in \cS$.
\end{assumption}

This assumption is essentially \cref{ass:precision lower bound} applied to all tasks. In general, for a uniform logging policy, $\gamma = \Omega(1 / d)$ when $n_s$ is large for all $s \in \cS$. Therefore, we do not think that the assumption is particularly strong. If needed, the assumption could be weaken to be probabilistic, as discussed after \cref{ass:precision lower bound}.

We also consider an additional assumption that sharpens the bound in \cref{thm:multi-task bound}.

\begin{assumption}
\label{ass:mab} For any $x \in \cX$ and $a \in \cA$, the feature vector $\phi(x, a)$ has at most one non-zero entry. Moreover, both $\Sigma_q$ and $\Sigma_0$ are diagonal.
\end{assumption}

Note that \cref{ass:mab} encompasses the multi-arm bandit case, where $\phi(x, a) \in \realset^{|\cX| |\cA|}$ and is an indicator vector for each context-action pair. Our main technical result is presented below.

\begin{theorem}
\label{thm:multi-task bound} Fix dataset $\cD$ and choose any $\gamma$ such that \cref{ass:multi-task precision lower bound} holds. Take $\hat{\pi}$ computed by \hieropo. Then for any $\delta \in (0, 1)$ and
\begin{align*}
  \alpha
  = \sqrt{5 d \log(1 / \delta)}\,,
\end{align*}
the suboptimality of $\hat{\pi}_s \in \Pi$ in \eqref{eq:multi-task suboptimality} is bounded for
\begin{align*}
  \varepsilon
  = {} & \underbrace{\alpha \sqrt{\frac{4 d}{\lambda_d(\Sigma_0^{-1}) +
  \gamma \sigma^{-2} n_s}}}_{\text{Task term}} + {} \\
  & \underbrace{\alpha \sqrt{\frac{4 d}{\lambda_d(\Sigma_q^{-1}) +
  \sum_{z \in \cS} \frac{1}{\lambda_1(\Sigma_0) +
  \gamma^{-1} \sigma^2 \lambda_1(G_{z, *}^{-1}) n_z^{-1}}}}}_{\text{Hyper-parameter term}}\,.
\end{align*}
Also, under \cref{ass:mab},
\begin{align*}
  \varepsilon
  = {} & \underbrace{\alpha \sqrt{\frac{4 d}{\lambda_d(\Sigma_0^{-1}) +
  \gamma \sigma^{-2} n_s}}}_{\text{Task term}} + {} \\
  & \underbrace{\alpha \sqrt{\frac{4 d}{\lambda_d(\Sigma_q^{-1}) +
  \sum_{z \in \cS} \frac{1}{\lambda_1(\Sigma_0) +
  \gamma^{-1} \sigma^2 n_z^{-1}}}}}_{\text{Hyper-parameter term}}\,.
\end{align*}
\end{theorem}
\begin{proof}
The claim is proved in \cref{sec:multi-task proof}, in the same three steps as \cref{thm:single-task bound}. The only difference is in the definitions of $c(x, a)$ and policies, and that we use \cref{ass:multi-task precision lower bound} instead of \cref{ass:precision lower bound}. This highlights the generality of our proof techniques and shows that they could be applicable to other graphical model structures.
\end{proof}

\subsection{Discussion}
\label{sec:discussion}

Our main technical result, an error bound on the suboptimality of policies learned by \hieropo, is presented in \cref{thm:multi-task bound}. The bound is Bayesian, meaning that it is proved for the distribution of true model parameters conditioned on logged dataset $\cD$. The bound has two terms. The former  captures the error in estimating the task parameter $\theta_{s, *}$ conditioned on known hyper-parameter $\mu_*$ and is analogous to \cref{thm:single-task bound}. We call it the \emph{task term}. The latter captures the error in estimating the hyper-parameter $\mu_*$ and we call it the \emph{hyper-parameter term}.

The task term scales with all quantities of interest as expected. First, it is $O(d \sqrt{\log(1 / \delta)})$, where $d$ is the number of task parameters and $\delta$ is the probability that the bound fails. This dependence is standard in linear bandit analyses with an infinite number of contexts \citep{abbasi-yadkori11improved,agrawal13thompson,abeille17linear}. Second, the task term decreases with the number of observations $n_s$ at the rate of $O(1 / \sqrt{n_s})$. Since $\lambda_d(\Sigma_0^{-1})$ can be viewed as the minimum number of prior pseudo-observations in any direction in $\realset^d$, the task term decreases with a more informative prior. Finally, the task term decreases when the observation noise $\sigma$ decreases, and the similarity of the logging and optimal policies $\gamma$ increases (\cref{ass:multi-task precision lower bound}).

The hyper-parameter term mimics the task-term scaling at the hyper-parameter level. In particular, the minimum number of prior pseudo-observations in any direction in $\realset^d$ becomes $\lambda_d(\Sigma_q^{-1})$ and each task becomes an observation, which is reflected by the sum over all tasks $z$. The hyper-parameter term decreases as the number of observations $n_z$ in any task $z$ increases, the maximum width of the task prior $\sqrt{\lambda_1(\Sigma_0)}$ decreases, noise $\sigma$ decreases, and the similarity between logging and optimal policies $\gamma$ increases.

To show that \hieropo leverages the structure of our problem, we compare its error bound to two baselines from \cref{sec:alternative designs}: \oracleopo and \flatopo. \oracleopo is an oracle estimator that knows $\mu_*$, meaning that it has more information than \hieropo. Its error is bounded in \cref{thm:single-task bound} and is always lower than that of \hieropo, as the error bound in \cref{thm:single-task bound} is essentially only the first term in \cref{thm:multi-task bound}. The second baseline, \flatopo, does not know $\mu_*$ and treats each task estimation problem independently. This approach can be viewed as \oracleopo where the task covariance $\Sigma_0$ is replaced by $\Sigma_q + \Sigma_0$, to account for the additional uncertainty due to not knowing $\mu_*$. The resulting error bound is
\begin{align*}
  \alpha \sqrt{\frac{4 d}{\lambda_d((\Sigma_q + \Sigma_0)^{-1}) +
  \gamma \sigma^{-2} n_s}}\,,
\end{align*}
and is always higher than the task term in \cref{thm:multi-task bound}. In addition, the hyper-parameter term in \cref{thm:multi-task bound} approaches zero as the number of tasks increases, and thus \hieropo is provably better in this setting of our interest.

The error bound in \cref{thm:multi-task bound} is proved for one fixed task $s \in \cS$. This decision was taken deliberately because other error bounds can be easily derived from this result. For instance, to get a bound for all tasks, we only need a union bound for the concentration of all $\theta_{s, *}$. Thus the bound in \cref{thm:multi-task bound} holds jointly for all $s \in \cS$ with probability at least $1 - m \delta$. Moreover, the same bound would essentially hold for any new task sampled from the hyper-prior. The reason is that the estimated hyper-parameter distribution, which affects the hyper-parameter term in \cref{thm:multi-task bound}, separates all other tasks from the evaluated one.

\section{Related Work}
\label{sec:related work}

\paragraph{Off-policy optimization.} In off-policy optimization, logged data collected by a deployed
policy is used to learn better policies \citep{li2010contextual}, and the agent does not interact with the environment directly. Off-policy learning can be achieved using model-free or model-based techniques. A popular model-free approach is empirical risk minimization with IPS-based estimators to account for the bias in logged data \citep{joachims2017unbiased,bottou2013counterfactual,swaminathan2015counterfactual,swaminathan2017off}. Model-based methods \citep{jeunen2021pessimistic} on the other hand learn a reward regression model for specific context-action pairs, which is then used to derive an optimal policy. Model-free methods tend to have a high variance while model-based methods tend to have a high bias unless explicitly corrected. Our approach is model based since we learn a hierarchical linear reward model.

\paragraph{Offline reinforcement learning.} The principle of pessimism has been explored in offline reinforcement learning in several works \citep{buckman2020importance,jin2021pessimism}. In particular, \citet{jin2021pessimism} show that pessimistic value iteration is minimax optimal in linear MDPs. The multi-task offline setting studied in this work was also studied by \citet{lazaric2010bayesian}. They propose an expectation-maximization algorithm but do not prove any error bounds. On the other hand, we consider a simpler setting of contextual bandits and derive error bounds that show improvemets due to using the multi-task structure.

\paragraph{Online learning.} Off-policy methods learn from data collected by a different policy. In contrast, online algorithms learn from data they collect, and need to balance exploration with exploitation. Two popular exploration techniques are upper confidence bounds (UCBs) \citep{auer02finitetime} and posterior sampling \citep{thompson33likelihood}, and they have been applied to linear reward models \citep{dani08stochastic,abbasi-yadkori11improved, chu11contextual,agrawal13thompson}. Bandit algorithms for hierarchical models have also been studied extensively \citet{bastani19meta,kveton21metathompson,basu21noregrets,simchowitz21bayesian,wan21metadatabased,hong22hierarchical,peleg22metalearning,wan22towards}. Perhaps surprisingly, all of these are based on posterior sampling. Our marginal posterior derivations in \cref{sec:hierarchical gaussian pessimism} can be used to derive their UCB counterparts.

\section{Experiments}
\label{sec:experiments}

In this section, we empirically compare \hieropo to baselines \oracleopo and \flatopo (\cref{sec:alternative designs}). All algorithms are implemented exactly as described in \cref{sec:algorithm} with $\alpha = 0.1$, which led to good performance in our initial experiments. Overall we aim to show that hierarchy can greatly improve the efficiency of off-policy algorithms.

\subsection{Synthetic Multi-Task Bandit}

\begin{figure*}[t]
  \centering
  \begin{minipage}{0.32\textwidth}
    \includegraphics[width=\linewidth]{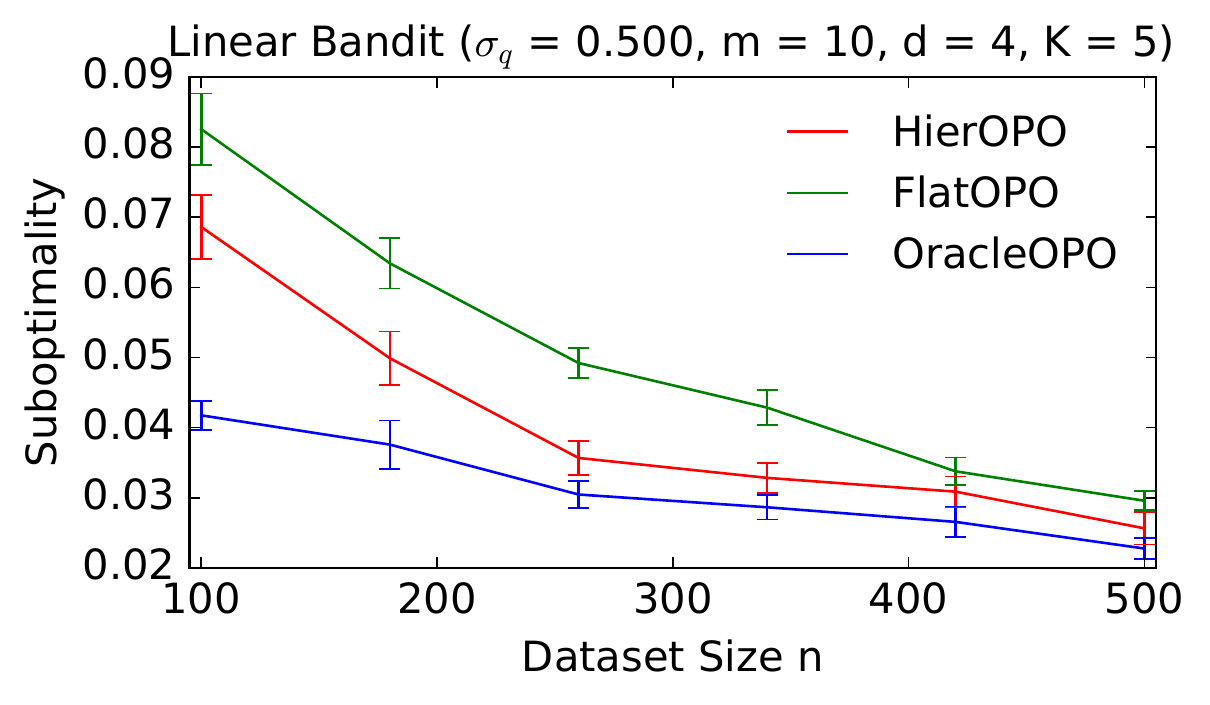}
  \end{minipage}
  \begin{minipage}{0.32\textwidth}
    \includegraphics[width=\linewidth]{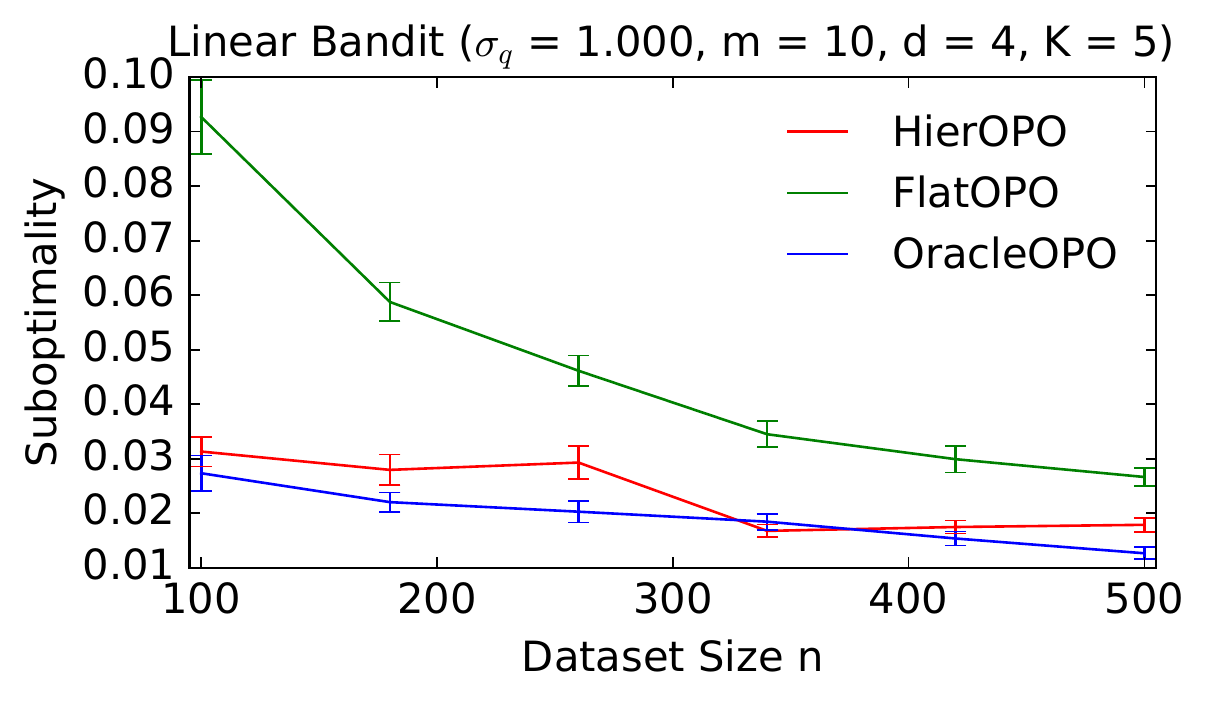}
  \end{minipage}
  \begin{minipage}{0.32\textwidth}
    \includegraphics[width=\linewidth]{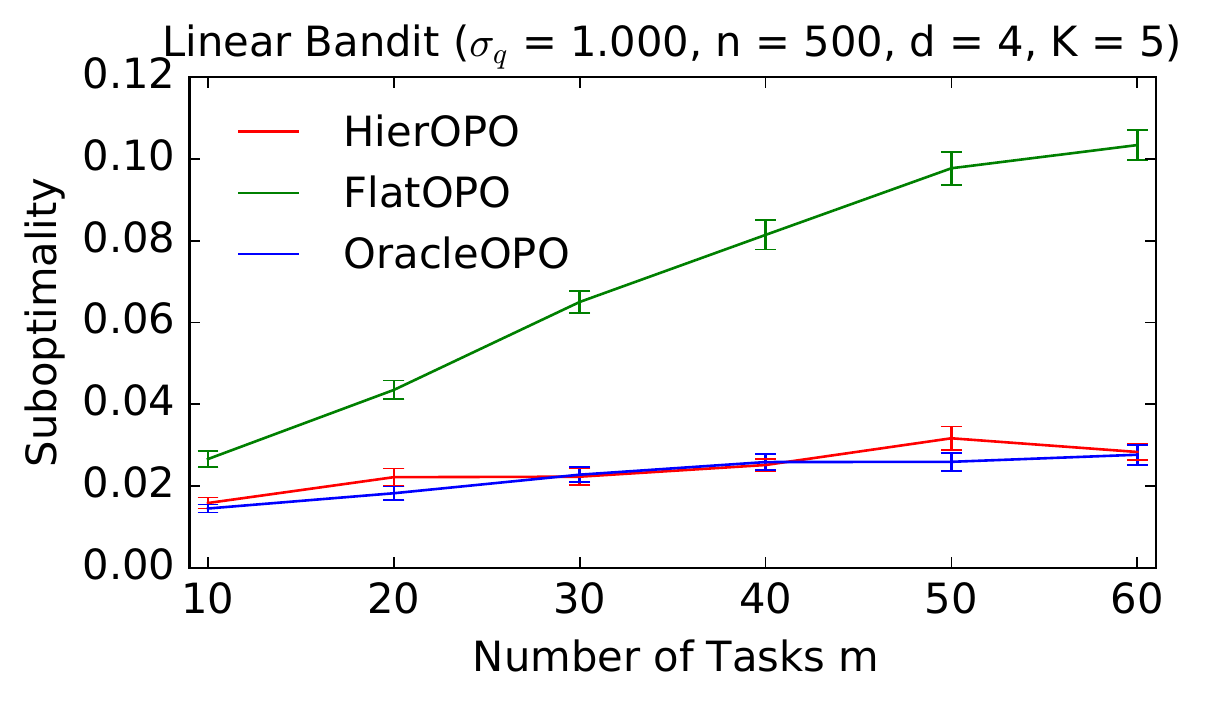}
  \end{minipage}
  \caption{Evaluation of off-policy algorithms on the synthetic multi-task bandit problem. In the left and middle plots, we vary the dataset size $n$ for small $\sigma_q = 0.5$ and large $\sigma_q = 1.0$. In the right plot, we vary the number of tasks $m$.}
  \label{fig:synthetic}
\vspace{-0.1in}
\end{figure*}

We first experiment with a synthetic multi-task bandit defined as follows. We set dimension as $d = 4$, number of actions as $K = 5$, and each context-action pair is a random vector $\phi(x, a) \in [-0.5, 0.5]^d$. The reward distribution for task $s$ is $\mathcal{N}(\phi(x, a)\T \theta_{s, *}, \sigma^2)$ with noise $\sigma = 0.5$.

The hierarchical model is defined as follows. The hyper-prior is $\cN(\mathbf{0}, \Sigma_q)$ with $\Sigma_q = \sigma_q^2 I_d$, the task covariance is $\Sigma_0 = \sigma_0^2 I_d$, and the reward noise is $\sigma = 0.5$. We choose $\sigma_q \in \set{0.5, 1}$ and $\sigma_0 = 0.5$. We expect more benefits of learning $\mu_*$ when $\sigma_q > \sigma_0$, as the uncertainty of the hyper-parameter is higher. The model parameters are generated as follows. At the beginning of each run, $\mu_* \sim \cN(\mathbf{0}, \Sigma_q)$. After that, each task parameter is sampled i.i.d.\ as $\theta_{s, *} \sim \cN(\mu_*, \Sigma_0)$. We initially set the number of tasks to $m = 10$ and the size of the logged dataset to $n = 500$. The logged dataset $\cD$ is generated as follows. For each interaction $t \in [n]$, we sample one of $m$ tasks uniformly at random, take an action uniformly at random, and sample a reward from the reward distribution.

In our experiments, we vary either dataset size $n$ or the number of tasks $m$ while keeping the other fixed. In \cref{fig:synthetic}, we show the mean and standard error of the suboptimality of each algorithm averaged over $30$ random runs, where the model and dataset in each run are generated as described earlier. As expected, \hieropo outperforms \flatopo and is close to \oracleopo. The improvement is greater when the uncertainty in the hyper-parameter $\sigma_q$ is higher. We also see that the gap is most noticeable in the limited data regime, where $n$ is small or $m$ is large, with only a small number of observations per task.

\subsection{Multi-User Recommendation}
\label{sec:multi-user recommendation}

\begin{figure}[t]
  \centering
  \includegraphics[width=.75\linewidth]{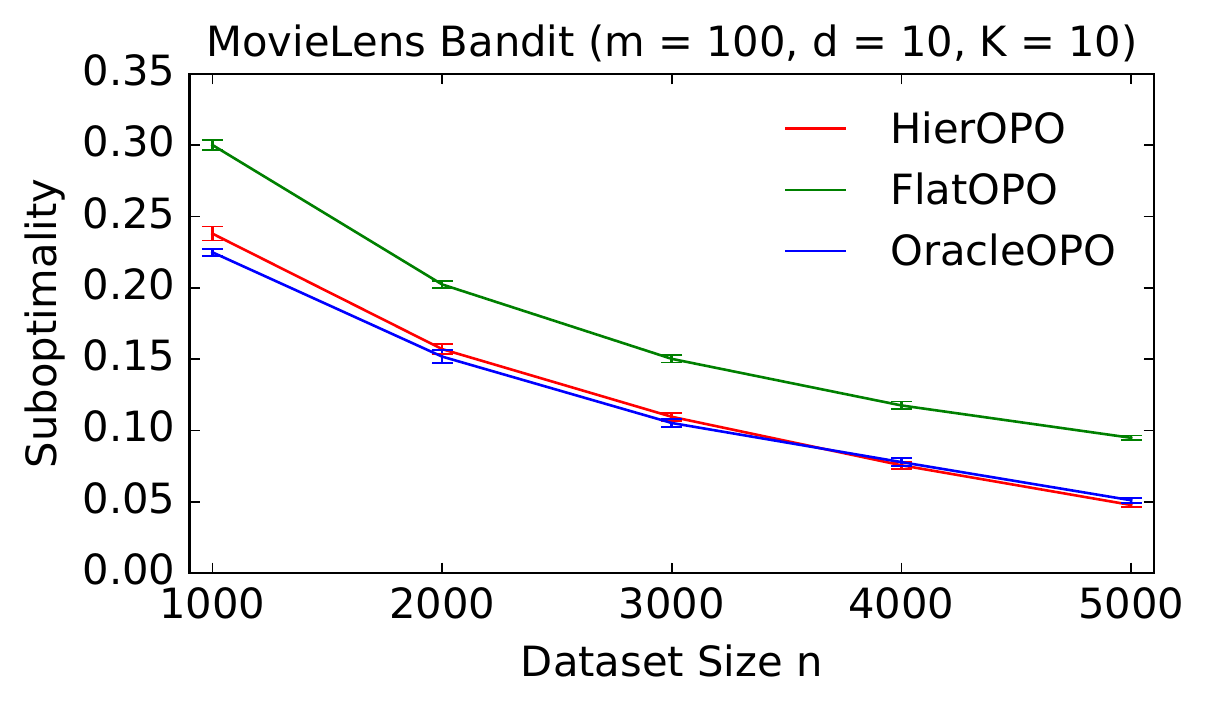}
  \caption{Evaluation of off-policy algorithms on the multi-user movie recommendation problem in \cref{sec:multi-user recommendation}.}
  \label{fig:movielens}
\vspace{-0.1in}
\end{figure}

Now we consider a multi-user recommendation application. We fit a multi-task contextual bandit from the MovieLens 1M dataset \citep{movielens}, with $1$ million ratings from $6\,040$ users for $3\,883$ movies, as follows. First, we complete the sparse rating matrix $M$ using alternating least squares \citep{pmf} with rank $d = 10$. This rank is high enough to yield a low prediction error, but small enough to avoid overfitting. The learned factorization is $M = U V\T$. User $i$ and movie $j$ correspond to rows $U_i$ and $V_j$, respectively, in the learned latent factors. Each task corresponds to some user $i$. In each round, context $x$ consists of $K = 10$ movies chosen uniformly at random. The reward distribution for recommending movie $j$ to user $i$ is $\mathcal{N}(V_j\T U_i, \sigma^2)$ with $\sigma = 0.759$ estimated from data.

To estimate the hierarchical model in \cref{sec:hierarchical gaussian pessimism}, we cluster the user latent factors. Specifically, we learn a \emph{Gaussian mixture model (GMM)} for $k = 7$ from rows of $U$, where we choose the smallest $k$ that still achieves low variance \citep{bishop06}. We estimate the hyper-prior parameters $\mu_q$ and $\Sigma_q$ using the mean and covariance, respectively, of the cluster centers. Then we select the cluster with most users, and set $\mu_*$ and $\Sigma_0$ to its center and covariance estimated by the GMM. The tasks are the users in this same cluster, to ensure that all are related to one another through the hyper-parameter. We wanted to stress that the GMM is only used to estimate parameters for the off-policy algorithms. The task parameters $U_i$ are generated by matrix factorization. This is to ensure that our setup is as realistic as possible.

We keep the number of tasks fixed at $m = 100$ and vary dataset size $n$. The tasks are users from the largest cluster, sampled uniformly at random. When generating the logged dataset, we sample one task uniformly at random, take a random action in it, and record its random reward. In \cref{fig:movielens}, we show the mean and standard error of the suboptimality of each algorithm averaged over $10$ random runs, where each run consists of choosing $m$ users, generating a dataset of size $n$, and running each algorithm on that dataset. We observe that \hieropo achieves good performance, close to \oracleopo, using much less data than \flatopo. This clearly demonstrates the benefit of hierarchies for statistically-efficient off-policy learning. The hierarchies are beneficial even if they are estimated from data and not exactly known.

\section{Conclusions}
\label{sec:conclusions}

In this work, we propose hierarchical off-policy optimization (\hieropo), a general off-policy algorithm for solving similar contextual bandit tasks related through a hierarchy. Our algorithm leverages the hierarchical structure to learn tighter, and thus more sample efficient, lower confidence bounds and then optimizes a policy with respect to them. We prove Bayesian suboptimality bounds for our policies, which decrease as the hyper-prior and task prior widths decrease. Thus the bounds improve with more informative priors. Finally, we empirically demonstrate the effectiveness of modeling hierarchies.

To the best of our knowledge, our work is the first to propose a practical and analyzable algorithm for off-policy learning with hierarchical Bayesian models. Because of this, there are many possible future directions to improve the generality and applicability of our approach. First, some applications may require more complex graphical models than two-level hierarchies. Second, the logged dataset may not contain labels of tasks, if different tasks cannot be as easily distinguished as users; or fully cover all possible tasks that can appear online. Extending our approach to learning policies from such limited datasets is another important avenue for future work.

\bibliographystyle{abbrvnat}
\bibliography{Brano,references}

\clearpage
\onecolumn
\appendix

\section{Appendix}
\label{sec:appendix}

This appendix contains proofs of our claims.

\subsection{Proof of \cref{thm:single-task bound}}
\label{sec:single-task proof}

The theorem proved using several lemmas. We start with the concentration of the model parameter. To simplify notation, we define $r(x, a) = r(x, a; \theta_*)$.

\begin{lemma}
\label{lem:concentration} Let
\begin{align*}
  E
  = \set{\forall x \in \cX, a \in \cA: \abs{r(x, a) - \hat{r}(x, a)} \leq c(x, a)}
\end{align*}
be the event that all high-probability confidence intervals hold. Then $\condprob{E}{\cD} \geq 1 - \delta$.
\end{lemma}
\begin{proof}
We start with the Cauchy–Schwarz inequality,
\begin{align*}
  r(x, a) - \hat{r}(x, a)
  = \phi(x, a) (\theta_* - \hat{\theta})
  = \phi(x, a) \hat{\Sigma}^{\frac{1}{2}}
  \hat{\Sigma}^{- \frac{1}{2}} (\theta_* - \hat{\theta})
  \leq \normw{\phi(x, a)}{\hat{\Sigma}}
  \normw{\theta_* - \hat{\theta}}{\hat{\Sigma}^{-1}}\,.
\end{align*}
Since $\theta_* - \hat{\theta} \sim \cN(\mathbf{0}, \hat{\Sigma})$, we know that $\hat{\Sigma}^{- \frac{1}{2}} (\theta_* - \hat{\theta})$ is a $d$-dimensional vector of i.i.d.\ standard normal variables. As a result, $(\theta_* - \hat{\theta})\T \hat{\Sigma}^{-1} (\theta_* - \hat{\theta})$ is a chi-squared random variable with $d$ degrees of freedom. Therefore, by Lemma 1 of \cite{laurent00adaptive},
\begin{align*}
  \delta
  & \geq \condprob{(\theta_* - \hat{\theta})\T \hat{\Sigma}^{-1} (\theta_* - \hat{\theta})
  \geq 2 \sqrt{d \log(1 / \delta)} + 2 \log(1 / \delta) + d}{\cD} \\
  & \geq \condprob{(\theta_* - \hat{\theta})\T \hat{\Sigma}^{-1} (\theta_* - \hat{\theta})
  \geq 5 d \log(1 / \delta)}{\cD} \\
  & = \condprob{\normw{\theta_* - \hat{\theta}}{\hat{\Sigma}^{-1}}
  \geq \sqrt{5 d \log(1 / \delta)}}{\cD}\,.
\end{align*}
This completes our proof.
\end{proof}

We use \cref{lem:concentration} to bound the suboptimality of $\hat{\pi}$ in any context by the confidence interval width induced by $\pi_*$.

\begin{lemma}
\label{lem:confidence intervals} The learned policy $\hat{\pi} \in \Pi$ satisfies
\begin{align*}
  r(x, \pi_*(x)) - r(x, \hat{\pi}(x))
  \leq 2 c(x, \pi_*(x))
\end{align*}
for all contexts $x \in \cX$ with probability at least $1 - \delta$.
\end{lemma}
\begin{proof}
For any context $x \in \cX$, we can decompose
\begin{align*}
  r(x, \pi_*(x)) - r(x, \hat{\pi}(x))
  & = r(x, \pi_*(x)) - L(x, \hat{\pi}(x)) +
  L(x, \hat{\pi}(x)) - r(x, \hat{\pi}(x)) \\
  & \leq r(x, \pi_*(x)) - L(x, \pi_*(x)) +
  L(x, \hat{\pi}(x)) - r(x, \hat{\pi}(x)) \\
  & = [r(x, \pi_*(x)) - L(x, \pi_*(x))] -
  [r(x, \hat{\pi}(x)) - L(x, \hat{\pi}(x))]\,.
\end{align*}
By \cref{lem:concentration}, event $E$ holds with probability at least $ 1 - \delta$. Under event $E$,
\begin{align*}
  r(x, \pi_*(x)) - L(x, \pi_*(x))
  = r(x, \pi_*(x)) - \hat{r}(x, \pi_*(x)) + c(x, \pi_*(x))
  \leq 2 c(x, \pi_*(x))\,.
\end{align*}
Analogously, under event $E$,
\begin{align*}
  r(x, \hat{\pi}(x)) - L(x, \hat{\pi}(x))
  = r(x, \hat{\pi}(x)) - \hat{r}(x, \hat{\pi}(x)) + c(x, \hat{\pi}(x))
  \geq 0\,.
\end{align*}
Now we combine the above two inequalities and get
\begin{align*}
  r(x, \pi_*(x)) - r(x, \hat{\pi}(x)) \leq 
  2 c(x, \pi_*(x))\,.
\end{align*}
This completes the proof.
\end{proof}

Since the above lemma holds for any context, we can use use it to bound the suboptimality of $\hat{\pi}$ by the expected confidence interval width induced by $\pi_*$,
\begin{align}
  V(\pi_*; \theta_*) - V(\hat{\pi}; \theta_*)
  & = \condE{r(X, \pi_*(X)) - r(X, \hat{\pi}(X))}{\theta_*}
  \leq 2 \condE{c(X, \pi_*(X))}{\theta_*}
  \label{eq:suboptimality to confidence intervals} \\
  & = 2 \sqrt{5 d \log(1 / \delta)} \
  \condE{\sqrt{\phi(X, \pi_*(X))\T \hat{\Sigma} \phi(X, \pi_*(X))}}{\theta_*}
  \nonumber \\
  & \leq 2 \sqrt{5 d \log(1 / \delta)} \
  \sqrt{\condE{\phi(X, \pi_*(X))\T \hat{\Sigma} \phi(X, \pi_*(X))}{\theta_*}}\,.
  \nonumber
\end{align}
The second inequality follows from the concavity of the square root.

The last step is an upper bound on the expected confidence interval width. Specifically, let $\Gamma = \Sigma_0^{-1} + \gamma \sigma^{-2} n G_*$. By \cref{ass:precision lower bound}, $\hat{\Sigma}^{-1} \succeq \Gamma$ and thus $\hat{\Sigma} \preceq \Gamma^{-1}$. So, for any policy $\pi_*$, we have
\begin{align*}
  \condE{\phi(X, \pi_*(X))\T \hat{\Sigma} \phi(X, \pi_*(X))}{\theta_*}
  & \leq \condE{\phi(X, \pi_*(X))\T \Gamma^{-1} \phi(X, \pi_*(X))}{\theta_*} \\
  & = \condE{\trace(\Gamma^{- \frac{1}{2}} \phi(X, \pi_*(X))
  \phi(X, \pi_*(X))\T \Gamma^{- \frac{1}{2}})}{\theta_*} \\
  & = \trace(\Gamma^{- \frac{1}{2}} G_*
  \Gamma^{- \frac{1}{2}}) \\
  & = \trace(G_* \Gamma^{-1})
  = \trace((\Sigma_0^{-1} G_*^{-1} + \gamma \sigma^{-2} n I_d)^{-1}) \\
  & \leq \frac{d}{\lambda_d(\Sigma_0^{-1} G_*^{-1} + \gamma \sigma^{-2} n I_d)}\,.
\end{align*}
The first inequality follows from \cref{ass:precision lower bound}. The first equality holds because $v\T v = \trace(v v\T)$ for any $v \in \realset^d$. The next three equalities use that the expectation of the trace is the trace of the expectation, the cyclic property of the trace, and the definition of matrix inverse. The last inequality follows from $\trace(A^{-1}) \leq d \lambda_1(A^{-1}) = d \lambda_d^{-1}(A)$, which holds for any PSD matrix $A \in \realset^{d \times d}$.

Now we apply basic eigenvalue identities and inequalities, and get
\begin{align*}
  \lambda_d(\Sigma_0^{-1} G_*^{-1} + \gamma \sigma^{-2} n I_d)
  & = \lambda_d(\Sigma_0^{-1} G_*^{-1}) + \gamma \sigma^{-2} n
  = \lambda_d((G_* \Sigma_0)^{-1}) + \gamma \sigma^{-2} n
  = \frac{1}{\lambda_1(G_* \Sigma_0)} + \gamma \sigma^{-2} n \\
  & \geq \frac{1}{\lambda_1(G_*) \lambda_1(\Sigma_0)} + \gamma \sigma^{-2} n
  \geq \frac{1}{\lambda_1(\Sigma_0)} + \gamma \sigma^{-2} n
  = \lambda_d(\Sigma_0^{-1}) + \gamma \sigma^{-2} n\,.
\end{align*}
To finalize the proof, we chain the last two claims and get
\begin{align*}
  \condE{\phi(X, \pi_*(X))\T \hat{\Sigma} \phi(X, \pi_*(X))}{\theta_*}
  \leq \frac{d}{\lambda_d(\Sigma_0^{-1}) + \gamma \sigma^{-2} n}\,.
\end{align*}
This completes the proof.

\subsection{Proof of \cref{thm:multi-task bound}}
\label{sec:multi-task proof}

The theorem is proved using several lemmas. We start with the concentration of the model parameter in task $s$. To simplify notation, let $r_s(x, a) = r(x, a; \theta_{s, *})$.

\begin{lemma}
\label{lem:multi-task concentration} Let
\begin{align*}
  E
  = \set{\forall x \in \cX, a \in \cA: \abs{r_s(x, a) - \hat{r}_s(x, a)} \leq c_s(x, a)}
\end{align*}
be the event that all high-probability confidence intervals in task $s \in \cS$ hold. Then $\condprob{E}{\cD} \geq 1 - \delta$.
\end{lemma}
\begin{proof}
The proof is analogous to \cref{lem:concentration}, since only the mean and covariance of $\theta_{s, *} \mid \cD$ changed, and this change is reflected in $\hat{r}_s(x, a)$ and $c_s(x, a)$.
\end{proof}

Now we apply \cref{lem:confidence intervals}, with task-dependent quantities and \cref{lem:multi-task concentration}, and get that the learned policy $\hat{\pi}_s$ satisfies
\begin{align*}
  r_s(x, \pi_{s, *}(x)) - r_s(x, \hat{\pi}_s(x))
  \leq 2 c_s(x, \pi_{s, *}(x))
\end{align*}
for all contexts $x \in \cX$ with probability at least $1 - \delta$. Since the above bound holds for any context, we can use use it to bound the suboptimality of $\hat{\pi}_s$ by the expected confidence interval width induced by $\pi_{s, *}$. Specifically, analogously to \eqref{eq:suboptimality to confidence intervals}, we have
\begin{align*}
  V(\pi_{s, *}; \theta_{s, *}) - V(\hat{\pi}_s; \theta_{s, *})
  & \leq 2 \condE{c_s(X, \pi_{s, *}(X))}{\theta_*} \\
  & \leq 2 \sqrt{5 d \log(1 / \delta)} \
  \sqrt{\condE{\phi(X, \pi_{s, *}(X))\T (\tilde{\Sigma}_s \Sigma_0^{-1} \bar{\Sigma}
  \Sigma_0^{-1} \tilde{\Sigma}_s + \tilde{\Sigma}_s) \phi(X, \pi_{s, *}(X))}{\theta_{s, *}}}\,.
\end{align*}
The latter term, which represents the conditional task uncertainty, can be bounded exactly as in \cref{thm:single-task bound},
\begin{align*}
  \condE{\phi(X, \pi_{s, *}(X))\T \tilde{\Sigma}_s \phi(X, \pi_{s, *}(X))}{\theta_{s, *}}
  \leq \frac{d}{\lambda_d(\Sigma_0^{-1}) + \gamma \sigma^{-2} n_s}\,.
\end{align*}
For the former term, which represents the hyper-parameter uncertainty, we have
\begin{align*}
  \condE{\phi(X, \pi_{s, *}(X))\T \tilde{\Sigma}_s \Sigma_0^{-1} \bar{\Sigma}
  \Sigma_0^{-1} \tilde{\Sigma}_s \phi(X, \pi_{s, *}(X))}{\theta_{s, *}}
  & \leq \trace(G_{s, *} \tilde{\Sigma}_s \Sigma_0^{-1} \bar{\Sigma}
  \Sigma_0^{-1} \tilde{\Sigma}_s) \\
  & \leq d \lambda_1(G_{s, *} \tilde{\Sigma}_s \Sigma_0^{-1} \bar{\Sigma}
  \Sigma_0^{-1} \tilde{\Sigma}_s)\,.
\end{align*}
To bound the maximum eigenvalue, we further proceed as
\begin{align*}
  \lambda_1(G_{s, *} \tilde{\Sigma}_s \Sigma_0^{-1} \bar{\Sigma}
  \Sigma_0^{-1} \tilde{\Sigma}_s)
  & \leq \lambda_1(G_{s, *}) \lambda_1(\tilde{\Sigma}_s \Sigma_0^{-1})
  \lambda_1(\bar{\Sigma}) \lambda_1(\Sigma_0^{-1} \tilde{\Sigma}_s) \\
  & \leq \lambda_1(\bar{\Sigma})
  = \frac{1}{\lambda_d(\Sigma_q^{-1} +
  \sum_{z \in \cS} (\Sigma_0 + G_z^{-1})^{-1})}\,.
\end{align*}
The second inequality follows from $\lambda_1(G_{s, *})\leq 1$ and $\lambda_1(\tilde{\Sigma}_s \Sigma_0^{-1}) \leq 1$. Finally, we apply basic eigenvalue identities and inequalities, and get
\begin{align*}
  \lambda_d\left(\Sigma_q^{-1} +
  \sum_{z \in \cS} (\Sigma_0 + G_z^{-1})^{-1}\right)
  & \geq \lambda_d(\Sigma_q^{-1}) +
  \sum_{z \in \cS} \lambda_d((\Sigma_0 + G_z^{-1})^{-1}) \\
  & = \lambda_d(\Sigma_q^{-1}) +
  \sum_{z \in \cS} \lambda_1^{-1}(\Sigma_0 + G_z^{-1}) \\
  & \geq \lambda_d(\Sigma_q^{-1}) +
  \sum_{z \in \cS} \frac{1}{\lambda_1(\Sigma_0) + \lambda_1(G_z^{-1})} \\
  & \geq \lambda_d(\Sigma_q^{-1}) +
  \sum_{z \in \cS} \frac{1}{\lambda_1(\Sigma_0) +
  \gamma^{-1} \sigma^2 \lambda_1(G_{z, *}^{-1}) n_z^{-1}}\,,
\end{align*}
where we use \cref{ass:multi-task precision lower bound} in the last inequality. When we combine the last three derivations, we get
\begin{align*}
  \condE{\phi(X, \pi_{s, *}(X))\T \tilde{\Sigma}_s \Sigma_0^{-1} \bar{\Sigma}
  \Sigma_0^{-1} \tilde{\Sigma}_s \phi(X, \pi_{s, *}(X))}{\theta_{s, *}}
  \leq \frac{d}{\lambda_d(\Sigma_q^{-1}) +
  \sum_{z \in \cS} (\lambda_1(\Sigma_0) +
  \gamma^{-1} \sigma^2 \lambda_1(G_{z, *}^{-1}) n_z^{-1})}\,.
\end{align*}
This completes the proof of the first claim in \cref{thm:multi-task bound}.

Note that the bound depends on $\lambda_1(G_{z, *}^{-1})$, which can be large when $\lambda_d(G_{z, *})$ is small. This is possible since $\pi_{z, *}$, which induces $G_{z, *}$, is a deterministic policy. We can eliminate this dependence when we adopt \cref{ass:mab}. Under this assumption, we have
\begin{align*}
  \lambda_1(G_{s, *} \tilde{\Sigma}_s \Sigma_0^{-1} \bar{\Sigma}
  \Sigma_0^{-1} \tilde{\Sigma}_s)
  = \lambda_1(G_{s, *} \bar{\Sigma} \tilde{\Sigma}_s \Sigma_0^{-1}
  \Sigma_0^{-1} \tilde{\Sigma}_s)
  \leq \lambda_1(G_{s, *} \bar{\Sigma})\,.
\end{align*}
The equality follows from the fact that all matrices in the product are diagonal and thus commute. Moreover.
\begin{align*}
  \lambda_1(G_{s, *} \bar{\Sigma})
  = \lambda_d^{-1}(\bar{\Sigma}^{-1} G_{s, *}^{-1})
  = \frac{1}{\lambda_d(\Sigma_q^{-1} G_{s, *}^{-1} +
  \sum_{z \in \cS} (G_{s, *} \Sigma_0 + G_{s, *} G_z^{-1})^{-1})}\,.
\end{align*}
Finally, we bound the minimum eigenvalue from below using basic eigenvalue identities and inequalities,
\begin{align*}
  \lambda_d\left(\Sigma_q^{-1} G_{s, *}^{-1} +
  \sum_{z \in \cS} (G_{s, *} \Sigma_0 + G_{s, *} G_z^{-1})^{-1}\right)
  & \geq \lambda_d(\Sigma_q^{-1}) \lambda_1^{-1}(G_{s, *}) +
  \sum_{z \in \cS} \lambda_1^{-1}(G_{s, *} \Sigma_0 + G_{s, *} G_z^{-1}) \\
  & \geq \lambda_d(\Sigma_q^{-1}) + \sum_{z \in \cS} \frac{1}
  {\lambda_1(G_{s, *}) \lambda_1(\Sigma_0) + \lambda_1(G_{s, *} G_z^{-1})} \\
  & \geq \lambda_d(\Sigma_q^{-1}) +
  \sum_{z \in \cS} \frac{1}{\lambda_1(\Sigma_0) + \gamma^{-1} \sigma^2 n_z^{-1}}\,.
\end{align*}
In the last two inequalities, we use that $\lambda_1(G_{s, *}) \leq 1$. In the last inequality, we also use that \cref{ass:multi-task precision lower bound} holds for any task parameter including $\theta_{z, *} = \theta_{s, *}$. Moreover, $G_z \succeq \gamma \sigma^{-2} n_z G_{s, *}$ implies $G_z^{-1} \preceq \gamma^{-1} \sigma^2 n_z^{-1} G_{s, *}$. This completes the proof of the second claim in \cref{thm:multi-task bound}.

\end{document}